\documentclass{article}
\usepackage{arxiv}
\usepackage[utf8]{inputenc}
\usepackage[T1]{fontenc}
\usepackage{amsmath,amssymb,amsthm}
\usepackage{booktabs,graphicx,microtype,tikz}
\usepackage[numbers]{natbib}
\usepackage{url,doi,hyperref,authblk}
\hypersetup{hidelinks}
\newtheorem{proposition}{Proposition}
\newtheorem{corollary}{Corollary}
\newtheorem{definition}{Definition}
\theoremstyle{remark}

\newcommand{\R}{\mathbb R}
\newcommand{\E}{\mathbb E}
\newcommand{\A}{\mathcal A}
\newcommand{\Ptwo}{\mathcal P_2(\R^d)}
\title{A Thermodynamic Theory of Learning Part II:\\
History-Dependent Reachability and Continual Learning}
\author{Daisuke Okanohara}
\affil{Preferred Networks, Inc.}
\date{Revised September 2026}
\renewcommand{\shorttitle}{History-Dependent Reachability}
\begin{document}
\maketitle

\begin{abstract}
We formulate
plasticity as target-dependent, finite-horizon reachability under
history-dependent dynamics, using standard minimum-energy control theory.
An extended state includes parameters and internal
variables that affect future updates. Around a reference trajectory, the
input-to-endpoint response and its controllability Gramian determine the
minimum input energy for each reachable displacement. Minimizing this energy
over the joint task target defines a task-conditioned adaptation cost, exact
for the arbitrary-input linear model. A solvable learning--retention example
shows that this cost can increase without any rank loss, and distinguishes
small directional gain from an inaccessible direction. Wasserstein transport
supplies a complementary global lower bound, while learning-map Jacobians
describe deformation of initial perturbations rather than response to future
inputs. The resulting learning--retention frontier depends on the current
extended state, admissible updates, target, horizon, and cost metric;
nonlinear and constrained-update applications require additional control of
these modeling choices.
\end{abstract}

\section{Introduction}
Continual learning asks a system to acquire new tasks while retaining earlier
performance. Its plasticity is often discussed in terms of how much capacity
remains. A scalar dimension or volume, however, cannot decide whether the
updates available to a learner point toward a particular task target, or
whether their gain is large enough to reach it within a finite budget.
We take target-dependent reachability under history-dependent dynamics as the
organizing viewpoint: the question is which joint task targets can be attained,
by which future updates, and at what cost.

Four questions must be separated. Static compatibility asks whether a common
solution exists. Geometric distance asks how far that solution set lies from
the current state. Deformation asks how earlier learning transformed initial
perturbations. Future accessibility asks how new inputs can change the endpoint.
In particular, initial-state sensitivity and future-input sensitivity are
different derivatives of the learning process:
\begin{equation}
 J_{\tau,0}=\frac{\partial\theta_\tau}{\partial\theta_0},
 \qquad K_{\tau,t}=\frac{\delta\theta_\tau}{\delta u_t}.
\end{equation}
The first derivative holds the future input protocol fixed; the second varies
that protocol from a fixed initial state. Neither can in general replace the other.

A minimal example shows why dimension is insufficient. Suppose the only
available update is $\dot x=u$, $\dot z=0$ from the origin. The targets
$(1,0)$ and $(0,1)$ have the same Euclidean distance and can be specified by
quadratic losses with identical Hessians, but only the first is reachable.
Even when both directions are available, $\dot x=a u_x$, $\dot z=u_z$ with
$a>0$ requires energy $1/(\tau a^2)$ to reach $(1,0)$ and $1/\tau$ to reach
$(0,1)$. The rank is two for every positive $a$, while the first cost diverges
as $a\downarrow0$. Task direction and directional gain are indispensable.

Learning history can change these gains through optimizer state, replay data,
or other internal variables even at identical parameter values. We therefore
start from an extended state and use its input-to-endpoint Gramian to define
a task-conditioned adaptation cost. The underlying minimum-energy identity
is standard linear control theory~\cite{BoydControl}; our use of it makes the
joint learning--retention target and the dependence on history explicit.
The contribution is a task-target formulation and a set of exact separation
results for continual learning. We minimize the classical input-energy cost
over a joint acquisition--retention target, derive its complete budget frontier
in an anisotropic two-dimensional model, and construct examples separating
compatibility, initial-state sensitivity, and future-input accessibility.
The minimum-energy identity, transport inequality, and matrix monotonicity
argument are established tools or direct consequences of them. We do not
claim a new controllability theorem, a universal mechanism of plasticity loss,
or an empirically validated predictor for neural-network training.

Part~I~\cite{okanohara2026} supplies a transport-cost and finite-time geometric
perspective. Here transport bounds serve as global necessary conditions, and
Jacobian volumes answer the separate question of initial-perturbation
deformation. A free-energy interpretation is restricted to the gradient-flow
model in Appendix~\ref{sec:freeenergy}. The main argument does not depend on
that interpretation. Appendix~\ref{app:limits} retains the counterexamples and
explicit withdrawals of earlier scalar-capacity claims. This version replaces
the capacity-threshold and unconditional monotonicity claims in versions~1--2
of this article~\cite{okanoharaPartIIv2}; the scope of the revision is substantive.

\subsection{Relation to prior work}
Empirical studies already distinguish the ability to acquire new tasks from
retention of previous ones. Dohare et al.~\cite{Dohare2024} demonstrate
plasticity loss in deep continual learning and study continual backpropagation
as an intervention. Lyle et al.~\cite{Lyle2023} connect plasticity loss to
loss-landscape curvature and investigate parameterization and optimization
choices. Our finite-dimensional examples do not reproduce those empirical
findings or establish their mechanism. They isolate a narrower mathematical
point: gain and target alignment affect finite-budget accessibility even when
rank and current parameters are unchanged.

Learning as optimal control is also an existing perspective. Vastola et
al.~\cite{Vastola2025} derive learning rules from trajectory optimization,
including effects of parameter-space geometry, partial controllability, and
adaptive state. Here we specify a future dynamics and admissible input class,
then ask whether its endpoints intersect a joint task target and what input
budget is required. This is an endpoint feasibility calculation; it does not
derive an optimal learning rule under uncertainty. Section~\ref{sec:diagnostics}
further relates this calculation to task Jacobians, NTK rank, and activation
subspaces. These comparisons locate the formulation within existing control
and plasticity research rather than treating its ingredients as new.

\section{Compatibility, Distance, Deformation, and Reachability}
\label{sec:concepts}
Let $\theta_A\in\R^d$ be the current parameters and let $\Phi_A,\Phi_B$ be
nonnegative task losses with $\Phi_A(\theta_A)<\infty$. For a retention tolerance $\varepsilon_A\ge0$ and
new-task threshold $b_B\ge0$, define the parameter target sets
\begin{equation}
 S_A=\{\theta:\Phi_A(\theta)\le\Phi_A(\theta_A)+\varepsilon_A\},\quad
 T_B=\{\theta:\Phi_B(\theta)\le b_B\},\quad C=S_A\cap T_B.
 \label{eq:parameter-targets}
\end{equation}
These conditions concern the endpoint. Requiring the entire learning path to
remain in $S_A$ is an additional constraint. Forgetting and failure to acquire
the new task are distinct outcomes: an inability to reach $C$ does not imply
forgetting unless new-task success is also required.

\begin{table}[t]
\centering\small
\begin{tabular}{@{}p{0.19\linewidth}p{0.28\linewidth}p{0.45\linewidth}@{}}
\toprule
Question & Mathematical object & What it establishes \\
\midrule
Static compatibility & Joint target $C=S_A\cap T_B$ & Whether any parameters meet both requirements. \\[4pt]
Geometric distance & Distance to $C$; distributional $W_2$ distance & A necessary transport budget, independent of the update law. \\[4pt]
Initial deformation & $J_{\tau,0}$ and $V_Q(J_{\tau,0})$ & How prescribed initial perturbations change under a fixed learning protocol. \\[4pt]
Future accessibility & $K_{\tau,t}$, $W_\tau$, reachable set & Which target directions future inputs can attain and at what input cost. \\
\bottomrule
\end{tabular}
\caption{Four distinct questions in continual learning. All response quantities
require a specified dynamics and horizon; distance also requires a metric.}
\label{tab:concepts}
\end{table}

For distributions, Section~\ref{sec:transport} uses expected-loss sets
$\mathcal S_A,\mathcal T_B,\mathcal C\subset\Ptwo$.
We reserve $C$ for parameter targets and $\mathcal C$ for distribution targets.
Dirac measures connect the two descriptions, but expected retention for an
ensemble does not imply retention by each member. In either setting, static
compatibility alone supplies no update rule that reaches the target.

\section{History-Dependent Reachability}
\label{sec:inputs}
\subsection{Extended state and future-input response}
Let the current state be $s_A=(\theta_A,h_A)$, where $h_A$ contains the internal
variables needed to specify future dynamics. Optimizer moments, normalization
statistics, or a fixed-dimensional representation of retained data may be
entered $h_A$. Consider a smooth finite-dimensional model
\begin{equation}
 \dot s=f_t(s,u),\qquad s(0)=s_A.
 \label{eq:extended}
\end{equation}
Discrete buffer replacements or architecture changes need a discrete or hybrid
model; a smooth representation is an assumption, not automatic.
Fix a reference input $\bar u$ and trajectory $\bar s(t)$, with parameter
endpoint $\bar\theta_\tau=P_\theta\bar s(\tau)$. Linearization gives
\begin{equation}
 \dot\xi=\mathsf A_t\xi+\mathsf G_t\delta u_t,\qquad \xi(0)=0,
 \quad \mathsf A_t=D_s f_t(\bar s(t),\bar u_t),\quad
 \mathsf G_t=D_u f_t(\bar s(t),\bar u_t).
\end{equation}
Let $U(t,r)$ be its state-transition matrix, and assume the resulting kernel
is square-integrable on $[0,\tau]$. The bounded endpoint operator and Gramian are
\begin{align}
 K_{\tau,t}&=P_\theta U(\tau,t)\mathsf G_t, &
 \mathcal L_\tau\delta u&=\int_0^\tau K_{\tau,t}\delta u_t\,dt,\nonumber\\
 W_\tau&=\mathcal L_\tau\mathcal L_\tau^*
       =\int_0^\tau K_{\tau,t}K_{\tau,t}^\top dt.
 \label{eq:gramian}
\end{align}
Here $y=\mathcal L_\tau\delta u$ is a displacement from
$\bar\theta_\tau$, not in general from $\theta_A$. Two histories with the same
$\theta_A$ can give different reference paths, kernels, and Gramians through
$h_A$. The reference input and the admissible input coordinates are part of
this specification.

\subsection{Minimum energy and the joint target}
\begin{proposition}[Minimum input energy]
\label{prop:energy}
In the linearized model with arbitrary signed $L^2$ inputs and energy
$\mathcal E[\delta u]=\int_0^\tau\|\delta u_t\|^2dt$, an endpoint displacement
$y$ is reachable exactly when $y\in\operatorname{Range}(W_\tau)$. Its minimum
energy is
\begin{equation}
 E_{\min}(y)=y^\top W_\tau^\dagger y,
 \qquad \delta u_t^*=K_{\tau,t}^\top W_\tau^\dagger y.
 \label{eq:minenergy}
\end{equation}
\end{proposition}
\begin{proof}
Because $W_\tau=\mathcal L_\tau\mathcal L_\tau^*$ has finite-dimensional
codomain, its range equals that of $\mathcal L_\tau$.
For $y$ in this range, $\mathcal L_\tau\delta u^*=W_\tau W_\tau^\dagger y=y$.
Every other solution is $\delta u^*+v$ with $v\in\ker\mathcal L_\tau$,
orthogonal to $\delta u^*\in\operatorname{Range}(\mathcal L_\tau^*)$.
Pythagoras gives minimality and
$\|\delta u^*\|_{L^2}^2=y^\top W_\tau^\dagger y$.
\end{proof}

\begin{definition}[Task-conditioned adaptation cost]
For the specified reference trajectory and input metric, define
\begin{equation}
 \mathcal C_\tau(C;s_A,\bar u)=
 \inf_{\substack{y\in\operatorname{Range}(W_\tau)\\
                  \bar\theta_\tau+y\in C}}
 y^\top W_\tau^\dagger y,
 \qquad \inf\varnothing=\infty.
 \label{eq:adaptation-cost}
\end{equation}
\end{definition}
When the reference parameters remain at $\theta_A$, the endpoint condition
reduces to $\theta_A+y\in C$. For a moving reference, the cost measures
additional input energy relative to that reference, not its total learning
effort. If the infimum is attained (in particular, for a closed target as
below), zero cost means that the reference endpoint already satisfies the
task requirements. Without attainment, zero cost can instead describe targets
arbitrarily close to that endpoint.

The linear reachable set under energy budget $E\ge0$ is
\begin{equation}
 \mathcal R_\tau^{\rm lin}(E)=\bar\theta_\tau+
 \{y\in\operatorname{Range}(W_\tau):y^\top W_\tau^\dagger y\le E\}.
 \label{eq:reachable}
\end{equation}
Thus $C\cap\mathcal R_\tau^{\rm lin}(E)\ne\varnothing$ is the exact budget
criterion. A cost greater than $E$ excludes attainment; a cost strictly below
$E$ implies it. At equality the infimum must be attained. If $C$ is closed,
as it is for continuous losses in \eqref{eq:parameter-targets}, every finite
infimum is attained: the feasible displacement set is closed and the
quadratic cost has compact sublevel sets on $\operatorname{Range}(W_\tau)$.
Under this condition, $\mathcal C_\tau(C;s_A,\bar u)\le E$ is equivalent to
attainment in the linear model. Without closedness, the open target $(1,\infty)$
with $W_\tau=1$ and reference endpoint zero has cost $1$ but is unreachable
at budget $1$.

For positive eigenvalues $\lambda_i$ and orthonormal eigenvectors $e_i$,
\begin{equation}
 E_{\min}(y)=\sum_{\lambda_i>0}\frac{(e_i^\top y)^2}{\lambda_i}
 \quad\text{for }y\in\operatorname{Range}(W_\tau).
\end{equation}
Outside the range the energy is infinite; the pseudoinverse formula alone
must not be used there. This expression exposes the role of task direction
and gain that a rank count omits. For an input cost
$\int\delta u_t^\top R_t\delta u_t\,dt$ with uniformly positive-definite,
bounded $R_t$, the corresponding Gramian is
$\int K_{\tau,t}R_t^{-1}K_{\tau,t}^\top dt$.
Costs are comparable only after the input metric has been fixed.

\subsection{Admissible updates and nonlinear scope}
The preceding result is exact for an affine system, or for the linearized
model considered on its own. For nonlinear dynamics an endpoint expansion
$\theta_\tau=\bar\theta_\tau+\mathcal L_\tau\delta u+r(\delta u)$ requires
bounds on the remainder and on the neighborhood in which it holds.
A target reached only on its boundary need not remain feasible under a small
endpoint error. Nonlinear drift or higher-order responses may also reach
directions absent from the linear range. Thus the cost is a local model
quantity, not an unconditional nonlinear accessibility theorem.

For fixed data-update vectors $v_i$ and nonnegative mixture rates
$p_i(t)$ satisfying $\sum_i p_i(t)\le1$, the endpoint displacement set is
$\tau\operatorname{conv}\{0,v_1,\ldots,v_m\}$. Integration gives membership,
and constant mixture rates attain every member. Unbounded nonnegative rates
instead give the generated cone; arbitrary signed rates give the linear
span before imposing an energy bound. Even full input rank cannot remove
these differences. State-dependent gradients require a corresponding
state-dependent input constraint.

A fixed loss imposes stronger restrictions still. For
$\Phi_B(\theta)=(c^\top\theta-1)^2/2$ and a fixed positive-definite
preconditioner $M$, gradient flow from zero has
\begin{equation}
 \theta(t)=\frac{Mc}{\alpha}(1-e^{-\alpha t}),\qquad
 \alpha=c^\top Mc>0.
 \label{eq:fixedgradient}
\end{equation}
This follows by solving the scalar residual equation
$\dot r=-\alpha r$, $r(0)=1$. The trajectory stays on
$\operatorname{span}(Mc)$ even when $M$ has full rank. For
$M=\operatorname{diag}(0.2,1)$ and $c=(1,1)$ it approaches $(1/6,5/6)$,
although the common solution $(1,0)$ for $\Phi_A=z^2/2$ is reachable by the
arbitrary-input system $\dot\theta=Mu$. Existence under arbitrary control is
not a guarantee for the fixed gradient trajectory. Section~\ref{app:input}
quantifies the analogous distinction between input cost and transport action.

\section{An Exact Learning--Retention Frontier}
\label{app:input}
The adaptation cost can be computed exactly when the future update gain is
fixed. This example compares the joint target with an input-cost ellipse;
$a$ may represent a gain set by history, without assuming a particular
mechanism that produced it. From $(x,z)=(0,0)$ let
\begin{equation}
 \dot x=a u_x,\quad \dot z=u_z,\quad 0<a\le1,\qquad
 \mathcal E[u]=\int_0^\tau(u_x^2+u_z^2)dt\le E,
\end{equation}
with arbitrary signed square-integrable inputs. The gain $a$ is fixed during
this phase and is not a past learning-map Jacobian. Let
$\Phi_A=\mu z^2/2$, $\mu>0$, and $\Phi_B=(x+z-1)^2/2$.
Success means $|x+z-1|\le r$, where $0\le r<1$; write $s=1-r$ and
$b=\tau E$. The exact endpoint set is
\begin{equation}
 x^2/a^2+z^2\le b.
 \label{eq:ellipse}
\end{equation}
Here $W_\tau=\tau\operatorname{diag}(a^2,1)$ and the reference is stationary.
Cauchy--Schwarz proves necessity, and constant inputs to each endpoint prove
sufficiency. The nearest target boundary is $x+z=s$: scaling any point farther
into the target strip toward zero to this boundary reduces both quadratic
cost and $|z|$.

On that boundary the cost is
$f_s(z)=(s-z)^2/a^2+z^2$, minimized at $z=s/(1+a^2)$. Thus
\begin{equation}
 b_{\mathrm{success}}=\frac{s^2}{1+a^2},\qquad
 b_{\mathrm{retain}}=\frac{s^2}{a^2}.
\end{equation}
Below the first value success is impossible. Between the two values success
is feasible but requires positive old-task loss. At or above the second,
success and perfect retention can coexist. In the intermediate interval,
solving $f_s(z)\le b$ gives
\begin{equation}
 z_{\min}=\frac{s-a\sqrt{(1+a^2)b-s^2}}{1+a^2},\qquad
 \min\Phi_A=\frac\mu2z_{\min}^2.
\end{equation}
The perfect-retention target therefore has
$\mathcal C_\tau(C)=s^2/(\tau a^2)$. Decreasing $a>0$ increases its adaptation
cost while preserving full Gramian rank. At $a=0$ the perfect-retention target
leaves the reachable subspace. Figure~\ref{fig:frontier} illustrates this
intersection geometry.

Both budget boundaries include equality. For tolerance $\Phi_A\le\varepsilon_A$,
put $\rho=\sqrt{2\varepsilon_A/\mu}$ and
$\bar z=\min\{s/(1+a^2),\rho\}$. The minimum joint budget is
$f_s(\bar z)$. Constant-input paths attain it; along those paths $\Phi_A$
increases monotonically to its endpoint value, so the same tolerance holds
throughout. This pathwise statement is specific to this example.

For the same velocity path, input effort and Euclidean action are
\begin{equation}
 \mathcal E=\int_0^\tau(\dot x^2/a^2+\dot z^2)dt,\qquad
 \A=\int_0^\tau(\dot x^2+\dot z^2)dt.
\end{equation}
If only $\A\le D$ is constrained, the success and perfect-retention thresholds
are instead $\tau D=s^2/2$ and $\tau D=s^2$, independently of $a>0$.
The gain-dependent threshold therefore arises from the input-cost metric.
At exactly $a=0$, $x$ cannot move and positive-target success with perfect
retention is impossible; the formula dividing by $a$ does not apply.

\begin{figure}[t]
\centering
\begin{tikzpicture}[scale=3.0,>=stealth]
 \fill[blue!8] (-.68,-.2) rectangle (1.24,.2);
 \fill[orange!10] (-.5,1.5)--(1.24,1.5)--(1.24,-.24)--cycle;
 \draw[blue!55,dashed] (-.68,.2)--(1.24,.2);
 \draw[blue!55,dashed] (-.68,-.2)--(1.24,-.2);
 \draw[->] (-.72,0)--(1.32,0) node[right] {$x$};
 \draw[->] (0,-1.25)--(0,1.55) node[above] {$z$};
 \draw[very thick,teal!80!black] (0,0) ellipse (.559 and 1.118);
 \draw[thick,orange!80!black] (-.5,1.5)--(1.24,-.24);
 \fill (.2,.8) circle (.022) node[right,xshift=3pt] {\small minimum-cost success};
 \fill (1,0) circle (.022) node[below,yshift=-15pt] {\small perfect retention};
 \draw[densely dotted] (.2,.8)--(.2,0);
 \node[anchor=west] at (.4,1.32) {\small $x+z\ge s$};
 \node[anchor=east] at (-.12,-.36) {\small $|z|\le\rho$};
 \node[anchor=west] at (.18,-.87) {\small reachable ellipse};
 \node[below left] at (0,0) {\small $0$};
 \node[above] at (1,0) {\small $(s,0)$};
\end{tikzpicture}
\caption{Reachability and task targets for $a=1/2$, $s=1$, $b=5/4$,
and retention tolerance $\rho=1/5$. The ellipse is the exact input-budget
endpoint set. The orange half-plane shows the lower success requirement
(the full task target is a strip); the blue band shows retention.
The minimum-cost success point is $(1/5,4/5)$, reached already at
$b_{\rm success}=4/5$. At the illustrated budget the ellipse meets success
but not joint retention. The perfect-retention point $(1,0)$ requires
$b_{\rm retain}=4$.}
\label{fig:frontier}
\end{figure}

\section{A Global Transport Lower Bound}
\label{sec:transport}
The input-cost frontier depends on the update law. A complementary obstruction
comes from the distance to the target in an unrestricted transport geometry.
This section uses a separate transport budget $D$, rather than the input
energy budget $E$ of Section~\ref{sec:inputs}.

Fix the Euclidean metric on $\R^d$ and a horizon $\tau>0$.
Let $q:[0,\tau]\to\Ptwo$ be a $2$-absolutely continuous Wasserstein curve,
and let $v_t$ be a Borel velocity field satisfying, in the distributional sense,
\begin{equation}
 \partial_t q_t+\nabla\cdot(q_tv_t)=0,\qquad
 \A[q,v]=\int_0^\tau\int\|v_t\|^2\,dq_t\,dt<\infty.
 \label{eq:action}
\end{equation}
Here $2$-absolute continuity means that an $L^2$ function bounds the metric
speed. Measures need not have densities; in particular, point trajectories
are represented by Dirac measures. No factor $1/2$ is included in $\A$.
The action belongs to the specified pair $(q,v)$; velocities representing the
same distribution path need not be unique.

We use the Benamou--Brenier dynamic formulation~\cite{Benamou2000,Villani2009},
\begin{equation}
 W_2(\mu,\nu)^2=\inf_{\rho,w}\int_0^1\int\|w_s\|^2\,d\rho_s\,ds,
 \label{eq:bb}
\end{equation}
where the infimum ranges over continuity-equation transports with endpoints
$\mu,\nu$. The following is the geometric content of the ESL.

\begin{proposition}[Finite-time transport bound]
\label{prop:esl}
Every transport in \eqref{eq:action} satisfies
\begin{equation}
 \A[q,v]\ge \frac{W_2(q_0,q_\tau)^2}{\tau}.
 \label{eq:esl}
\end{equation}
\end{proposition}
\begin{proof}
Set $\rho_s=q_{\tau s}$ and $w_s=\tau v_{\tau s}$ on $[0,1]$.
This pair satisfies the continuity equation and has action $\tau\A[q,v]$.
Equation~\eqref{eq:bb} gives the result.
\end{proof}
The bound is attained by a constant-speed Wasserstein geodesic when all
transports are allowed. Such a geodesic need not be generated by a fixed loss,
a specified optimizer, or a prescribed stochastic process.

\subsection{Endpoint retention and new-task success}
Let $\Phi_A,\Phi_B:\R^d\to[0,\infty]$ be Borel losses and write
$L_i(q)=\int\Phi_i\,dq$. Fix a current distribution $q_A$ with
$L_A(q_A)<\infty$, a retention tolerance $\varepsilon_A\ge0$, and a finite
new-task threshold $b_B\ge0$. Define
\begin{align}
 a_A&=L_A(q_A)+\varepsilon_A, &
 \mathcal S_A&=\{q\in\Ptwo:L_A(q)\le a_A\},\nonumber\\
 \mathcal T_B&=\{q\in\Ptwo:L_B(q)\le b_B\}, &
 \mathcal C&=\mathcal S_A\cap\mathcal T_B.
 \label{eq:targets}
\end{align}
These are expected-loss requirements; they do not imply retention by every
individual ensemble member. A loss sublevel set is also not generally a smooth
task-preserving manifold.

\begin{proposition}[Necessary budget for joint attainment]
\label{prop:joint}
Put $\delta_{\mathcal C}=\inf_{q\in\mathcal C}W_2(q_A,q)$, with
$\inf\varnothing=\infty$. If a transport starts at $q_A$, has
$\A[q,v]\le D<\infty$, and ends in $\mathcal C$, then
\begin{equation}
 \delta_{\mathcal C}^2\le\tau D.
 \label{eq:joint}
\end{equation}
Thus $\delta_{\mathcal C}^2>\tau D$ excludes joint attainment.
\end{proposition}
\begin{proof}
By the definition of the infimum and Proposition~\ref{prop:esl},
$\delta_{\mathcal C}^2\le W_2(q_A,q_\tau)^2\le\tau\A[q,v]\le\tau D$.
\end{proof}
\begin{corollary}[Retention failure conditional on success]
\label{cor:forgetting}
If $\delta_{\mathcal C}^2>\tau D$, every transport within budget that ends in
$\mathcal T_B$ has $L_A(q_\tau)>a_A$.
\end{corollary}
\begin{proof}
Otherwise its endpoint would be in $\mathcal C$, contradicting
Proposition~\ref{prop:joint}.
\end{proof}
The corollary does not assert the existence of a successful transport. Failure
to reach $\mathcal T_B$ is another possible outcome. Nor does it give a positive
uniform lower bound on the amount by which $L_A$ exceeds $a_A$.

If $\mathcal C=\varnothing$, the requirements are statically incompatible.
If $\mathcal C\ne\varnothing$ but \eqref{eq:joint} fails, the obstruction is
relative to the chosen time and action budget. Increasing either can change
the answer. No compactness or attainment assumption is needed for the
necessary condition. Conversely, if a distance-minimizing endpoint in
$\mathcal C$ exists and all transports are allowed, its geodesic gives endpoint
feasibility whenever \eqref{eq:joint} holds. At equality, nonattainment of the
infimum can prevent feasibility. These observations do not give sufficiency
for a specified optimizer.

\subsection{Pathwise retention is a separate constraint}
Requiring $q_t\in\mathcal S_A$ for every $t$ is stronger than endpoint
retention. For example, take $\Phi_A(x)=x^2$, $q_A=\delta_0$, and
$\varepsilon_A=0$. The path $q_t=\delta_{\sin^2(\pi t/\tau)}$ returns to
$\mathcal S_A$ at $\tau$ but violates retention at every interior time.
Proposition~\ref{prop:joint} remains necessary under pathwise retention, but
its geodesic converse need not respect that constraint.

\subsection{Global distance and dynamics-specific feasibility}
The two tests answer different questions:
\begin{equation}
 \delta_{\mathcal C}^2\le\tau D,
 \qquad \mathcal C\cap\mathcal R_\tau^{\rm dyn}(D)\ne\varnothing,
\end{equation}
where $\mathcal R_\tau^{\rm dyn}(D)$ denotes distribution endpoints generated
by the specified dynamics within transport action $D$. The first is necessary
for the second. It measures how far the target is; the second asks whether
the dynamics can move in the required directions. For example, from the
origin the target $(0,1)$ has squared Dirac transport distance $1$, but is
unreachable under $\dot x=u$, $\dot z=0$ for every budget.
When updates are instead constrained by input energy, the relevant set is
$\mathcal R_\tau^{\rm lin}(E)$ or its nonlinear counterpart. A relation between
input energy and transport action must be established before combining their
numerical budgets; Section~\ref{app:input} shows why they need not agree.

\section{Deformation of Initial Perturbations}
\label{sec:volume}
The preceding sections concern new inputs applied from the current state.
We now consider a different experiment: perturb the initial state while
holding the learning protocol fixed. Its Jacobian measures deformation of
prescribed initial perturbations, not remaining capacity or future plasticity.
For an extended-state flow, a parameter-to-parameter derivative must also
specify which initial internal variables are held fixed. The square flow maps
below act on the full state, or on parameters when they form a closed system;
a projected parameter block need not inherit full-state invertibility.

Let $\Psi_{t,s}$ be $C^1$ interval maps with
$\Psi_{t,0}=\Psi_{t,s}\circ\Psi_{s,0}$. At
$x_s=\Psi_{s,0}(x_0)$ their Jacobians satisfy
\begin{equation}
 J_{t,0}(x_0)=J_{t,s}(x_s)J_{s,0}(x_0).
 \label{eq:chain}
\end{equation}
Task switching and time dependence require this two-time notation; a
one-parameter semigroup is not assumed. A probability-flow map, an optimal
transport map, and a stochastic update map conditioned on a noise realization
can have different Jacobians even when their endpoint distributions agree.

\begin{definition}[Normalized volume factor]
For a fixed $Q\in\R^{d\times k}$, $Q^\top Q=I_k$, $1\le k\le d$, define
\begin{equation}
 G_Q(J)=Q^\top J^\top JQ,\qquad
 V_Q(J)=\det G_Q(J)^{1/(2k)},\qquad R_Q(J)=V_Q(J)^2.
 \label{eq:volume}
\end{equation}
\end{definition}
The actual $k$-volume multiplier is $\sqrt{\det G_Q}$; $V_Q$ is its $k$th
root. For $J=I$, $V_Q=R_Q=1$ independently of $k$; for $J=cI$,
$V_Q=|c|$ and $R_Q=c^2$. These are scale factors, not counts of directions.
A zero determinant gives zero volume factor. If $Q$ spans an initial
retention-manifold tangent space, $JQ$ need not lie in a retention tangent
space at the endpoint. Such an interpretation additionally requires the
map to preserve the manifold.

\begin{proposition}[Conditional volume monotonicity]
\label{prop:volume}
For $J,K\in\R^{d\times d}$ and the same $Q$ on both sides, if
$K^\top K\preceq I$, then
\begin{equation}
 G_Q(KJ)\preceq G_Q(J),\qquad V_Q(KJ)\le V_Q(J).
\end{equation}
\end{proposition}
\begin{proof}
Put $M=JQ$. Then $M^\top K^\top KM\preceq M^\top M$.
Eigenvalue monotonicity for positive semidefinite matrices implies the same
ordering for their determinants. Taking nonnegative $2k$th roots proves the
claim, including singular cases.
\end{proof}
For random maps the condition must hold for each realization, along the
properly composed interval maps. With $X=\log\det G_Q(J)$, $\log0=-\infty$,
and $\E[X^+]<\infty$, the expectation lies in $[-\infty,\infty)$.
The pointwise inequality then also implies monotonicity of
$\exp(\E[X]/k)$, the averaged squared-volume quantity used previously.
The same initial $Q$ must be used in each comparison.

Convex gradient flows supply one sufficient condition: a variation obeys
$\dot\xi=-\nabla^2\Phi\,\xi$, so
$d\|\xi\|^2/dt=-2\xi^\top\nabla^2\Phi\xi\le0$.
Nonconvex flows and discrete updates do not generally satisfy this condition.
Non-expansiveness is sufficient, not necessary, for volume decrease.
For instance, $K=\operatorname{diag}(2,1/4)$ decreases full volume but expands
one direction.

\subsection{Rank, log volume, and finite-time invertibility}
Always $\operatorname{rank}(KJ)\le\min(\operatorname{rank}K,
\operatorname{rank}J)$. This concerns propagation of initial perturbations.
For square invertible matrices the exact identity is
\begin{equation}
 \log\det((KJ)^\top KJ)
 =\log\det(J^\top J)+\log\det(K^\top K).
\end{equation}
There are no additional interaction terms. The corresponding restricted
volume does not in general have this simple additive formula.

For a smooth flow with $\dot J_t=B_tJ_t$, $J_0=I$, and
$\int_0^\tau\|B_t\|\,dt<\infty$, the fundamental matrix is invertible and
$\det J_t=\exp(\int_0^t\operatorname{tr}B_s\,ds)>0$.
Thus finite-time exact rank loss does not occur under these assumptions.
Small singular values, noninvertible discrete updates, and singular limits
must be distinguished.

For $B_t=\nabla v_t(x_t)$ and $M_t=J_tQ$ of full column rank,
\begin{equation}
 \frac{d}{dt}\log\det(M_t^\top M_t)
 =2\operatorname{tr}(P_t\operatorname{Sym}B_t),\quad
 P_t=M_t(M_t^\top M_t)^{-1}M_t^\top.
 \label{eq:strain}
\end{equation}
Indeed, differentiate $M_t^\top M_t$ and use
$d\log\det G/dt=\operatorname{tr}(G^{-1}\dot G)$.
For $k=d$ this becomes $2\nabla\cdot v_t(x_t)$.
Action involves $\|v_t\|^2$ whereas volume change involves its spatial
 derivative; neither expression supplies a general inequality relating them.

\section{History-Dependent Plasticity}
\label{sec:history}
The dependence of $W_\tau$ on $s_A=(\theta_A,h_A)$ makes it possible to ask
how two learning histories change adaptation from the same parameters.
Keep the task target, horizon, reference input protocol, and input metric
fixed. For histories $h_1,h_2$, the inequality
\begin{equation}
 \mathcal C_\tau(C;(\theta_A,h_2),\bar u)
 >\mathcal C_\tau(C;(\theta_A,h_1),\bar u)
\end{equation}
means that the second history makes this target more costly in the specified
response model. If the reference endpoints differ, this comparison includes
both reference drift and response geometry; a comparison intended to isolate
the Gramian should also match the reference endpoint.

Plasticity degradation in this sense need not remove a direction. In
Section~\ref{app:input}, a smaller positive gain increases the cost of the
retention-preserving displacement while the rank stays two. More generally,
the projection of the target displacement onto weak Gramian eigenvectors
can raise cost. If two Gramians are positive definite, share a reference
endpoint, and satisfy $W_\tau^{(2)}\preceq W_\tau^{(1)}$, inverse ordering gives
$y^\top (W_\tau^{(2)})^{-1}y\ge y^\top (W_\tau^{(1)})^{-1}y$ for every $y$ and hence the same
ordering of adaptation costs for a fixed target. Without such an ordering,
one history can help one target and hinder another. A single scalar ranking
of histories then loses task-relevant information.

A simple extended-state realization makes this dependence exact. Set
$\dot x=(1+h)^{-1}u_x$, $\dot z=u_z$, and $\dot h=0$, with $h\ge0$.
The same parameters $(x,z)=(0,0)$ and different stored values of $h$ give
$W_\tau(h)=\tau\operatorname{diag}((1+h)^{-2},1)$.
For the perfect-retention target in Section~\ref{app:input}, the adaptation
cost is $s^2(1+h)^2/\tau$. Under zero reference input, the parameter
initial-condition Jacobian is the identity for every $h$. Thus identical
parameters and identical initial-state deformation can coexist with different
future adaptation costs. This is a model of a stored gain, not a derivation
of a particular optimizer's state dynamics.

Optimizer moments and preconditioners affect the state evolution and input
coupling. Resetting optimizer state can therefore change $\mathsf A_t$,
$\mathsf G_t$, and the reference endpoint at fixed $\theta_A$; it does not
imply a universal improvement. Replay changes which updates are admissible
and may also change the reference objective. Regularization changes drift
and its linearization, potentially trading target progress against retention.
Normalization statistics can alter the same operators. Architecture changes
can change the state space itself, requiring a specified correspondence
between task outputs before costs are compared. Stochastic updates require
a choice between conditioning on a noise realization and probabilistic
reachability; the deterministic Gramian alone supplies no probability of
success. These mechanisms fit the framework through explicit model changes,
not through a common monotonicity claim.

\section{Discussion}
The operational quantity developed here is the minimum future input energy
needed to reach a joint target in a specified response model. Its value is
relative to the target, current extended state, reference protocol, horizon,
and input metric. Static compatibility is a prerequisite; global transport
distance provides another necessary obstruction. Neither specifies the
admissible future learning directions. Initial-perturbation deformation is
a complementary sensitivity experiment, with its own conditional volume
result. This separation prevents any one geometry from being asked to answer
all four questions in Table~\ref{tab:concepts}.

\subsection{Relation to plasticity diagnostics}
\label{sec:diagnostics}
The Jacobian of model outputs with respect to parameters differs from the
Jacobian of a learning map with respect to initial parameters. Under a local
output linearization with endpoint output Jacobian $H$, an input response
$\mathcal L_\tau$ induces output response $H\mathcal L_\tau$ and output Gramian
$HW_\tau H^\top$. This gives one explicit way to connect a feature or output
geometry to a task target, while retaining the future update dynamics.

Jacobian sketches of earlier data can support provable retention in specified
linear and wide-network models~\cite{Heckel2022}. NTK-rank diagnostics in
continual reinforcement learning~\cite{Tang2025} and readout-relevant activation
subspaces~\cite{Anthes2025} supply other task-linked geometries. Feature rank
and activation subspaces describe representation structure; a Hessian spectrum
can enter the drift linearization for a gradient flow; gradient norm describes
an instantaneous update magnitude. None alone specifies the finite-horizon
input response, its alignment with a joint target, and its cost metric.
Their use as plasticity diagnostics therefore needs assumptions connecting
them to the response and target under study.

\subsection{What the formulation does and does not predict}
The exact example and minimum-energy identity establish a target-specific
cost interpretation, not a claim that a Gramian diagnostic predicts progress
for every optimizer. Data-dependent gradients impose input constraints, and
long nonlinear trajectories can invalidate a local approximation.
A useful empirical test should compare histories at matched initial loss,
task difficulty, horizon, and budget, specify the admissible updates, and
measure both acquisition and retention. It should test whether the proposed
response approximation predicts endpoints or adaptation cost on independent
tasks, with uncertainty across training histories. The exploratory negative
result retained in Appendix~\ref{app:limits} makes this validation requirement
concrete.

\section{Conclusion}
Continual-learning plasticity is usefully formulated as target-dependent
reachability under history-dependent dynamics and finite cost. A joint
solution can exist without being accessible, and a full-rank response can
still make a task prohibitively costly. The input-to-endpoint Gramian and its
task-conditioned adaptation cost express these distinctions exactly in an
arbitrary-input linear model. Transport distance gives a global necessary
budget, while learning-map volume describes initial-perturbation deformation.
History-dependent plasticity loss can thus appear as increased adaptation
cost without rank disappearance; extending this account to a learner requires
its admissible updates and approximation errors to be specified.

\section*{Acknowledgments}
We thank Akihito Sudo for careful and constructive comments on an earlier
version of this manuscript. His counterexamples identified problems in the
previous capacity-threshold and monotonicity arguments and enabled us to
correct the formulation substantially. His comments also helped us isolate
the non-expansiveness assumption used in
Proposition~\ref{prop:volume} and clarify the distinction between deformation
of initial perturbations and reachability under future inputs.

\section*{Use of generative AI tools}
OpenAI Codex was used to assist with manuscript restructuring and drafting,
mathematical derivations and proof exposition, verification-code preparation,
and the preparation of the figure. Responsibility
for the scientific claims, citations, and final submitted content rests with
the author.

\bibliographystyle{plain}
\bibliography{references}

\appendix
\section{Free-Energy Interpretation}
\label{sec:freeenergy}
The action in \eqref{eq:action} is a geometric cost. Identifying it with
physical entropy production, optimizer input effort, or computation requires
a separate model and appropriate coefficients. For comparison with the
Wasserstein gradient-flow formulation~\cite{Jordan1998}, assume in this section
that $q_t$ is a smooth positive density, all displayed integrals are finite,
and differentiation and integration by parts are valid with vanishing boundary
terms. Let $T_t,m_t>0$ be smooth and spatially constant, and set
\begin{equation}
 F_t[q]=\int\Phi_t q\,d\theta-T_tH[q],\quad
 H[q]=-\int q\log q\,d\theta,\quad
 v_t=-m_t\nabla(\Phi_t+T_t\log q_t).
\end{equation}
Differentiating and using the continuity equation gives
\begin{align}
 \frac{d}{dt}F_t[q_t]
 &=\E_{q_t}[\partial_t\Phi_t]-\dot T_tH[q_t]
   +\int q_t\nabla(\Phi_t+T_t\log q_t)\cdot v_t\,d\theta\nonumber\\
 &=\E_{q_t}[\partial_t\Phi_t]-\dot T_tH[q_t]
   -\frac1{m_t}\int\|v_t\|^2\,dq_t.
 \label{eq:balance}
\end{align}
For fixed objective, temperature, and constant mobility $m$, integration
therefore gives $F[q_0]-F[q_\tau]=\A/m$. Only at unit mobility does this equal
$\A$. For external driving the first two terms in \eqref{eq:balance} remain.
For an arbitrary velocity field the last term need not have a negative sign,
because the gradient-flow relation is not assumed.

Consequently, a free-energy drop is not generally the minimum action between
endpoints. We use the unambiguous geometric excess
\begin{equation}
 \A_{\mathrm{ex}}[q,v]=\A[q,v]-\frac{W_2(q_0,q_\tau)^2}{\tau}\ge0,
 \label{eq:excess}
\end{equation}
relative to unrestricted transport. It is not defined as action minus
free-energy drop. Rescaling the same path from unit duration to duration
$\tau$ divides velocity and action by $\tau$; maintaining the same free energy
in the gradient-flow equation then rescales mobility as well.

\section{Limits of Scalar Capacity Interpretations}
\label{app:limits}
This appendix collects the negative results and scope qualifications that
motivate the conditional formulation used in the main text. They are grouped
here so that the paper's main argument can proceed from the claims it
establishes.

\subsection{The withdrawn compatible-capacity threshold}
An earlier formulation called an exponentiated average log determinant an
``effective rank'' and compared it with the stable rank of a new-task Hessian.
The resulting Compatible Capacity Threshold is not valid. Consider
\begin{equation}
 \Phi_A(x,y,z)=\tfrac12z^2,\qquad
 \Phi_B(x,y,z)=\tfrac12\bigl((x-a)^2+(y-b)^2\bigr).
\end{equation}
Task-$A$ gradient flow for duration $T$ has Jacobian
$J_T=\operatorname{diag}(1,1,e^{-T})$. On the preserving plane $z=0$,
$Q=(e_1,e_2)$ and $R_Q(J_T)=1$, while the restricted Hessian of $B$ is $I_2$
and has stable rank $2$. Thus the old threshold is triggered, even at $T=0$.
Nevertheless, starting at $(x_0,y_0,0)$, task-$B$ gradient flow is
\begin{equation}
 (x(t),y(t),z(t))=
 (a+e^{-t}(x_0-a),\ b+e^{-t}(y_0-b),\ 0).
\end{equation}
It has $\Phi_A=0$ for the whole path and
$\Phi_B(t)=e^{-2t}\Phi_B(0)$. Every positive target accuracy is attained in
finite time, with action $\Phi_B(0)(1-e^{-2\tau})$.

A dimension count alone cannot repair the threshold. With updates restricted
to $\dot x=u$, $\dot y=\dot z=0$ from the origin, the tasks
$((x-1)^2+y^2)/2$ and $(x^2+(y-1)^2)/2$ have identical restricted Hessians
$I_2$. The first can be solved along the available line, while the second has
loss at least $1/2$ on that line. Target direction and permitted updates are
therefore indispensable.

\subsection{Limits of volume monotonicity}
Unconditional volume monotonicity under composition also fails. In one
dimension, $J=1/2$ followed by $K=4$ changes $R$ from $1/4$ to $4$. These maps
are gradient flows of $x^2/2$ for time $\log2$ and $-x^2/2$ for time $\log4$,
respectively. The second map is expansive, which is precisely
the case excluded by Proposition~\ref{prop:volume}.

For a smooth finite-time flow, the fundamental matrix remains invertible under
the integrability assumptions stated in Section~\ref{sec:volume}. Hence a
small volume or small singular value should not be described as exact rank
loss without an additional singular limit or noninvertible update. Moreover,
the learning-map Jacobian concerns initial-state perturbations, whereas future
adaptation is governed by the admissible input response developed in
Section~\ref{sec:inputs}.

\subsection{No universal relation between excess action and contraction}
Positive excess action need not produce contraction. For any
$q_0\in\Ptwo$, a unit vector $e$, and $a(t)=\sin^2(\pi t/\tau)$, set
$\Psi_t(x)=x+a(t)e$ and $q_t=(\Psi_t)_\#q_0$. Then $J_t=I$ for all $t$,
$q_\tau=q_0$, and
\begin{equation}
 \A_{\mathrm{ex}}=\A=\int_0^\tau\dot a(t)^2dt
 =\frac{\pi^2}{2\tau}>0.
\end{equation}
This is an admissible translation loop with no deformation.

Conversely, contraction can occur at zero excess action. Take
$q_0=N(0,I_d)$, $0<c<1$, and
$\Psi_t(x)=(1+(c-1)t/\tau)x$. The action is $d(1-c)^2/\tau$.
For any endpoint coupling $(X,Y)$, Cauchy--Schwarz gives
$\E[X\cdot Y]\le cd$ and hence $\E\|X-Y\|^2\ge d(1-c)^2$.
The coupling $Y=cX$ attains this value, so
$W_2(q_0,q_\tau)^2=d(1-c)^2$ and $\A_{\mathrm{ex}}=0$, although the Jacobian
contracts throughout. Thus an inference from excess action alone to
contraction or forgetting requires additional dynamical assumptions.

\subsection{Interpretive and empirical scope}
A budget-dependent reachability boundary describes a finite-horizon
limitation after the budget, target, and permissible dynamics have been
defined. Irreversible critical-period closure or a thermodynamic phase
transition would require further evidence. In particular, vanishing gain and
small positive gain differ because the latter can permit adaptation at a
larger budget or under a changed optimizer. Old-task degradation and new-task
plasticity likewise remain separate observables.

Appendix~\ref{app:pilot} reports a
small-network pilot comparing retained and reset RMSProp state at identical
parameters. Its tested one-step response did not outperform the gradient-norm
baseline for predicting twenty-step progress across the five reported seeds.
This is neither a validation nor a refutation of the full finite-horizon
Gramian formulation: the pilot measures an instantaneous scalar proxy under
a fixed update law, not the minimum energy for a joint target. The results
are included to delimit the empirical claims, not as support for the propositions.

The earlier Compatible Capacity Threshold, unconditional volume monotonicity,
and universal dissipation-to-forgetting inference are therefore withdrawn.

\section{Reproducible Exploratory Pilot}
\label{app:pilot}
\subsection{Protocol and measured quantities}
We use a scalar-output network
$F_\theta(x)=b_0+\sum_{j=1}^8 v_j\tanh(w_jx+b_j)$ on the 21 equally spaced
inputs $x_i=-1+2i/20$, $i=0,\ldots,20$. Synthetic task targets are
\begin{equation}
 t_\alpha(x)=\cos(\alpha)x+\sin(\alpha)
 \left(x^2-\frac1{21}\sum_{i=0}^{20}x_i^2\right),\qquad
 L_\alpha(\theta)=\frac1{42}\sum_{i=0}^{20}
       (F_\theta(x_i)-t_\alpha(x_i))^2.
\end{equation}
Angles below are specified in degrees and converted to radians in the code.
For each seed, initial weights are sampled independently and uniformly from
$w_j\in[-0.8,0.8]$, $b_j\in[-0.2,0.2]$, and $v_j\in[-0.5,0.5]$, with
$b_0=0$ and zero RMSProp accumulator. Training uses full-batch gradients and
\begin{equation}
 g_k=\nabla L_\alpha(\theta_k),\quad
 h_{k+1}=\beta h_k+(1-\beta)g_k^{\odot2},\quad
 \theta_{k+1}=\theta_k-\eta g_k\oslash\sqrt{\epsilon+h_{k+1}},
\end{equation}
where operations on vectors are elementwise, $\eta=0.01$, $\beta=0.99$,
and $\epsilon=10^{-4}$ is inside the square root. Each training task lasts
80 updates, with 50 training angles sampled uniformly from $[0,180)$.
Checkpoints are taken after 0, 5, 20, and 50 tasks.

At each checkpoint, each evaluation angle in
$\{12,37,68,97,133,169\}$ is learned for 20 updates from a fresh copy of the
checkpoint, either retaining $h$ or resetting it to zero. Parameters match
exactly between each pair. Evaluation angles are fixed separately from the
random training angles. There are 48 conditions per seed and five seeds:
11, 29, 103, 2027, and 314159. Conditions within a seed share a training
history and are not independent replicates.

With initial evaluation gradient $g$, loss $L_0$, and
$\widetilde h=\beta h+(1-\beta)g^{\odot2}$, the compared predictors are
\begin{equation}
 p_{\rm response}=\frac{\eta}{L_0}
   \sum_j\frac{g_j^2}{\sqrt{\epsilon+\widetilde h_j}},\qquad
 p_{\rm gradient}=\frac{\eta}{L_0}\|g\|^2.
\end{equation}
The first is a first-order predicted fractional decrease for one specified
RMSProp update; it is not an estimated finite-horizon Gramian.
The baseline is also normalized by initial loss. Outcomes are the realized
fractional improvements $(L_0-L_k)/L_0$ for $k=1,20$.
An additional diagnostic is the output-Jacobian stable rank
$\|J_F\|_F^2/\|J_F\|_2^2$, with the largest eigenvalue of $J_FJ_F^\top$
approximated by power iteration. This Jacobian is not the learning-map Jacobian.

\subsection{Results and limitations}
Table~\ref{tab:pilot} gives Pearson correlations pooled over the 48 conditions
within each seed. The gradient baseline exceeds the response proxy at 20
steps in all five seeds. The one-step response correlation is also negative
for seed 2027. These observations do not support a claim of robust predictive
superiority for this proxy. They do not establish that directional response
information is useless or that a finite-horizon response estimate would fail.

\begin{table}[ht]
\centering\small
\begin{tabular}{r r r r r}
\toprule
Seed & Response, 1 step & Response, 20 steps & Gradient, 20 steps & Rank, 20 steps \\
\midrule
11     & 0.6481 & 0.2619 & 0.5650 & 0.0388 \\
29     & 0.9231 & 0.5112 & 0.7512 & 0.0983 \\
103    & 0.8770 & 0.5723 & 0.7549 & -0.0926 \\
2027   & -0.6731 & 0.1965 & 0.5248 & -0.3573 \\
314159 & 0.7481 & 0.5224 & 0.7754 & 0.2035 \\
\bottomrule
\end{tabular}
\caption{Exploratory within-seed pooled Pearson correlations. Each row pools
four checkpoints, six evaluation tasks, and two optimizer-state conditions.
The entries are descriptive statistics, not independent-sample significance tests.}
\label{tab:pilot}
\end{table}

This small tanh regression system need not exhibit the plasticity loss seen
in deeper continual learners. Pooling can confound checkpoint, initial loss,
and task difficulty. The six task angles are related synthetic functions,
and this exploratory comparison is not a preregistered benchmark. No
confidence intervals or population-level significance claims are made.
The exported data also include a change in loss on an anchor task at angle
zero; that anchor is not a joint retention constraint enforced during updates.
Consequently the pilot does not test the joint-target adaptation-cost theorem.
A test of that account would need a specified realizable input class, a
validated response approximation, and acquisition and retention endpoints
measured against the same target and budget.

\end{document}